\documentclass[sigconf]{acmart}

\usepackage{colortbl}
\usepackage{hyperref}
\usepackage{url}
\usepackage{xspace,latexsym}
\usepackage{thm-restate}
\usepackage{multirow}
\usepackage{comment}
\usepackage{algorithm}
\usepackage{algorithmic}
\usepackage{wrapfig}
\usepackage{graphicx}
\usepackage{caption}
\usepackage{subcaption}
\usepackage{pbox}
\usepackage{bm}
\usepackage{afterpage}
\usepackage{booktabs}

\begin{document}

\title{Efficient and Effective Implicit Dynamic Graph Neural Network}

\author{Yongjian Zhong}

\affiliation{%
\department{Department of Computer Science}
  \institution{University of Iowa}
  \city{Iowa City}
  \state{IA}
  \country{USA}
}
\email{yongjian-zhong@uiowa.edu}

\author{Hieu Vu}

\affiliation{%
\department{Department of Computer Science}
  \institution{University of Iowa}
  \city{Iowa City}
  \state{IA}
  \country{USA}
}
\email{hieu-vu@uiowa.edu}

\author{Tianbao Yang}

\affiliation{%
\department{Department of Computer Science and Engineering}
  \institution{Texas A\&M University}
  \city{College Station}
  \state{TX}
  \country{USA}
}
\email{tianbao-yang@tamu.edu}

\author{Bijaya Adhikari}

\affiliation{%
\department{Department of Computer Science}
  \institution{University of Iowa}
  \city{Iowa City}
  \state{IA}
  \country{USA}
}
\email{bijaya-adhikari@uiowa.edu}

\newcommand{\fix}{\marginpar{FIX}}
\newcommand{\new}{\marginpar{NEW}}
\newcommand{\ourModel}{\textsc{IDGNN}\xspace}
\newcommand{\GT}{\mathcal{G}}
\newtheorem{prop}{Property}
\newtheorem{assumption}{Assumption}
\newtheorem{observ}{Observation}
\newtheorem{remark}{Remark}
\definecolor{Gray}{gray}{0.94}

\begin{abstract}

Implicit graph neural networks have gained popularity in recent years as they capture long-range dependencies while improving predictive performance in static graphs. Despite the tussle between performance degradation due to the oversmoothing of learned embeddings and long-range dependency being more pronounced in dynamic graphs, as features are aggregated both across neighborhood and time, no prior work has proposed an implicit graph neural model in a dynamic setting. 

In this paper, we present Implicit Dynamic Graph Neural Network (\ourModel) a novel implicit neural network for dynamic graphs which is the first of its kind. A key characteristic of \ourModel is that it demonstrably is well-posed, i.e., it is theoretically guaranteed to have a fixed-point representation. We then demonstrate that the standard iterative algorithm often used to train implicit models is computationally expensive in our dynamic setting as it involves computing gradients, which themselves have to be estimated in an iterative manner. 
To overcome this, we pose an equivalent bilevel optimization problem and propose an efficient single-loop training algorithm that avoids iterative computation by maintaining moving averages of key components of the gradients. We conduct extensive experiments on real-world datasets on both classification and regression tasks to demonstrate the superiority of our approach over the state-of-the-art baselines. We also demonstrate that our bi-level optimization framework maintains the performance of the expensive iterative algorithm while obtaining up to \textbf{1600x} speed-up. 
\end{abstract}



\maketitle

\section{Introduction}


Graph Convolution Network (GCN)~\citep{kipf2016semi} and its subsequent variants~\citep{li2018adaptive}~\cite{velickovic2018graph} have achieved the state-of-the-art performance in predictive tasks in various applications including molecular prediction \citep{park2022acgcn}, recommendation \citep{liao2022sociallgn}, and hyperspectral image classification \citep{hong2020graph}. GCNs have also been extended to the dynamic setting, where the graph changes over time. Even in the dynamic setting, GCNs have achieved state-of-the-art results for numerous tasks including rumor detection \citep{sun2022ddgcn} and traffic prediction \citep{li2023dynamic}.

Despite numerous advantages, a major limitation of existing GCNs is that they can only aggregate information up to $k$-hops, where $k$ is the depth of the graph convolution operation. Hence, standard graph neural networks~\cite{kipf2016semi,velickovic2018graph} cannot capture long-range dependencies beyond a radius imposed by the number of convolution operations used. Trivial solutions like setting $k$ to a large number fail in overcoming this issue as empirical evidence~\cite{li2018deeper} suggests that deepening the layers of GCN, even beyond a few (2-4) layers, can lead to a notable decline in their performance. This is because the stacked GCN layers gradually smooth out the node-level features which eventually results in non-discriminative embeddings (aka oversmoothing). This creates a dilemma where, on the one hand, we would like to capture dependencies between nodes that are far away in the network by stacking multiple layers of GCN together. On the other hand, we also would like to maintain the predictive performance by only using a few layers. To tackle this dilemma in the static setting, Gu et al. \cite{gu2020implicit} proposed an implicit graph neural network (IGNN), which iterates the graph convolution operator until the learned node representations converge to a \emph{fixed-point} representation. Since there is no a priori limitation on the number of iterations, the fixed-point representation potentially contains information from all neighbors in the graph. Evidence shows that it is able to capture long-range dependency while maintaining predictive performance. Following this, other recent works~\cite{park2021convergent, liu2021eignn} also have addressed the problem in the static setting.


In the case of dynamic graphs, where graphs evolve over time, dynamic graph neural networks aggregate information over the current graph topology and historical graphs to learn meaningful representations~\cite{sun2022ddgcn,yang2020featurenorm}. Note the architecture of the graph neural networks within each time stamp in dynamic graph neural networks is similar to existing static GCNs. Hence, the information aggregated for each node in a given time stamp is still limited to a radius of $k-hops$ within the time stamp. Increasing the depth of the GCN operator in each time stamp to ensure long-range dependency exacerbates the performance degradation of dynamic graph neural networks as these models convolve both over the time stamps and within the time stamps, hence oversmoothing the features even faster. Therefore, capturing long-range dependency while improving (or even maintaining) the performance is a big challenge for dynamic graph neural networks. Despite its importance, very few prior works have studied this phenomenon in dynamic graphs: Yang et al. \cite{yang2020featurenorm} propose an L2 feature normalization process to alleviate the smoothing of features in dynamic graphs and Wang et al. \cite{wang2022tpgnn} mitigate the oversmoothing problem by emphasizing the importance of low-order neighbors via a node-wise encoder. However, these approaches either rescale features or forget neighborhood information, both of which are not ideal.

To address the challenges mentioned above, we propose \ourModel, an implicit neural network for dynamic graphs derived from the first principles. In designing \ourModel, we encountered multiple challenges including \emph{i)} uncertainty on whether fixed-point (converged) representations exist for implicit neural models defined over dynamic graphs; and \emph{ii)} efficiently training a model to find these fixed-point representations. In this paper, we overcome the first challenge by providing theoretical guarantees on the existence of the fixed-point representations on a single dynamic graph by leveraging a periodic model and generalizing this result to a set of dynamic graphs. For the second challenge, we notice that the stochastic gradient descent via implicit differentiation (often used by other implicit models~\cite{gu2020implicit,li2022unbiased}) is too inefficient in our problem setting. As such, we reformulate our problem as an equivalent bilevel optimization problem and design an efficient optimization strategy. The key contributions of the paper are as follows:

\begin{itemize}
    \item We propose a novel dynamic graph neural network \ourModel, which ensures long-range dependency while providing theoretical guarantees on the existence of fixed-point representations. \ourModel
    is the first approach to leverage implicit graph neural network framework for dynamic graphs. 
    \item We present a bilevel optimization formulation of our problem and propose a novel stochastic optimization algorithm to efficiently train our model. Our experiments show that the proposed optimization algorithm is faster than the naive gradient descent by up to 1600 times.
    \item We conduct comprehensive comparisons with existing methods to demonstrate that our method captures the long-range dependency and outperforms the state-of-the-art dynamic graph neural models on both classification and regression tasks.
\end{itemize}


\section{Related Work}
\textbf{Dynamic Graph Representation Learning:}
GNN has been successful for static graphs, leading to the development of GNN-based algorithms for dynamic graphs \citep{khoshraftar2022survey}. DyGNN \citep{ma2020streaming} comprises two components: propagation and update, which enable information aggregation and propagation for new interactions. EvolveGCN \citep{pareja2020evolvegcn} uses an RNN to update GCN parameters and capture dynamic graph properties.  Sankar et. al. \cite{sankar2020dysat} propose a Dynamic Self-Attention Network (DySAT) with structural and temporal blocks to capture graph information. TGN \citep{rossi2020temporal} models edge streaming to learn node embeddings using an LSTM for event memory. TGAT \citep{xu2020inductive} considers the time ordering of node neighbors. Gao et al. \cite{gao2022equivalence} explores the expressiveness of temporal GNN models and introduces a time-then-graph framework for dynamic graph learning, leveraging expressive sequence representations like RNN and transformers.

\textbf{Implicit Graph Models:}
The implicit models or deep equilibrium models define their output using fixed-point equations. \cite{bai2019deep}. propose an equilibrium model for sequence data based on the fixed-point solution of an equilibrium equation. El et al. \cite{el2021implicit} introduce a general implicit deep learning framework and discuss the well-posedness of implicit models. Gu et al. \cite{gu2020implicit} demonstrate the potential of implicit models in graph representation learning, specifically with their implicit model called IGNN, which leverages a few layers of graph convolution network (GCN) to discover long-range dependencies. Park et al. \cite{park2021convergent} introduce the equilibrium GNN-based model with a linear transition map, and they ensure the transition map is contracting such that the fixed point exists and is unique. Liu et al. \cite{liu2021eignn} propose an infinite-depth GNN that captures long-range dependencies in the graph while avoiding iterative solvers by deriving a closed-form solution. Chen et al. \cite{chen2022optimization} employ the diffusion equation as the equilibrium equation and solve a convex optimization problem to find the fixed point in their model.

\textbf{Implicit Models Training:}
Efficiently training implicit models has always been a key challenge.  Normally, the gradient of implicit models is obtained by solving an equilibrium equation using fixed-point iteration or reversing the Jacobian matrix \citep{gu2020implicit}. However, training these models via implicit deferential introduces additional computational overhead. Geng et al. \cite{geng2021training} propose a phantom gradient to accelerate the training of implicit models based on the damped unrolling
and Neumann series. Li et al.
\cite{li2022unbiased} leverage
 stochastic proximal gradient descent and its variance-reduced version to accelerate the training.

\section{Methodology}
\begin{figure}[t]
    \centering
    \includegraphics[width=0.5\textwidth]{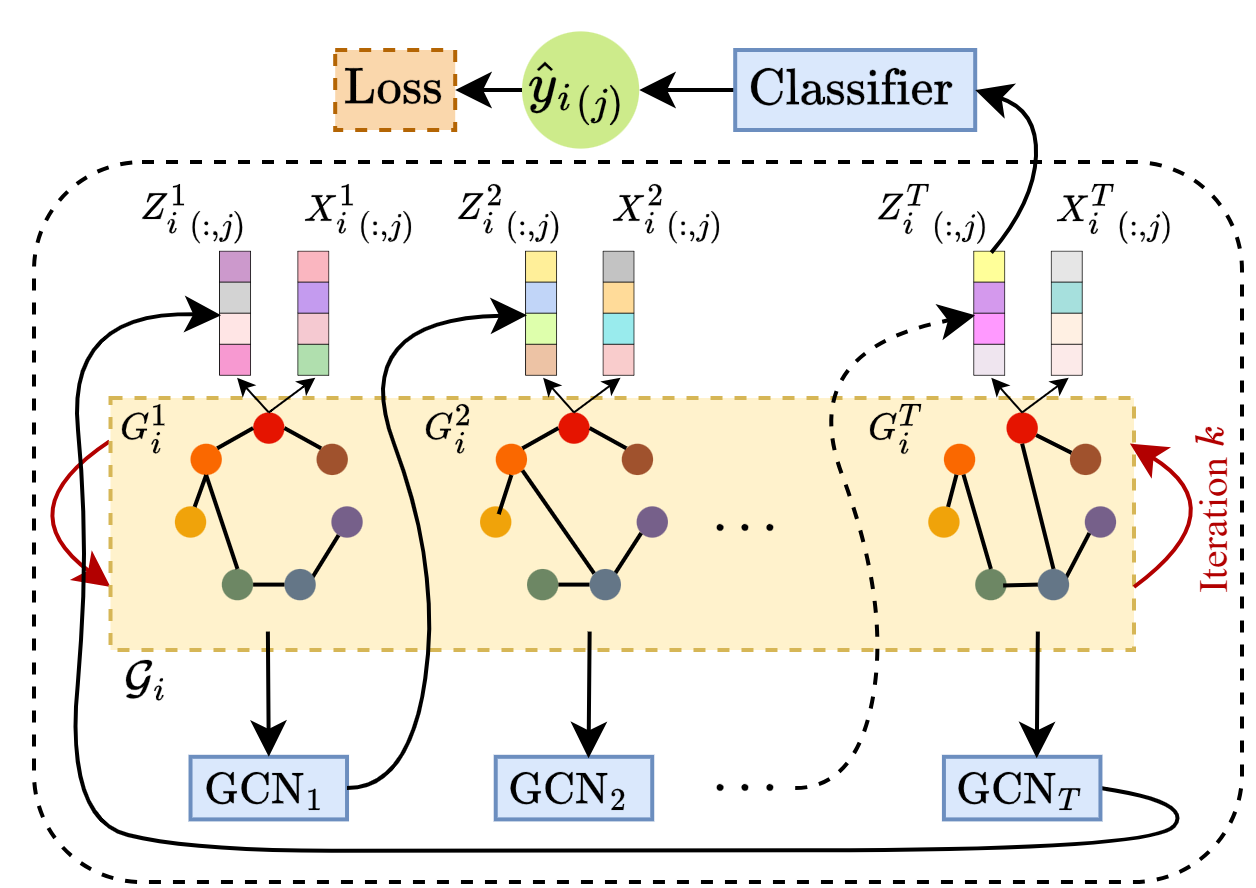}
    \caption{Model overview. This figure indicates the forward process of our model.}
\end{figure}
\subsection{Preliminaries}

\noindent\textbf{Dynamic Graphs:}
We are given a set of $N$ dynamic graphs $\{\mathcal{G}_i\}_{i=1}^N$. Each dynamic graph $\mathcal{G}_i=\{G_i^1,...,G_i^t,...,G_i^T\}$ is a collection of $T$ snapshots. Let $\mathcal{V}$ denote the union set of nodes that appear in any dynamic graph and $n:=|\mathcal{V}|$ be the total number of nodes. Without loss of generality, each snapshot can be represented as $G_i^t=\{\mathcal{V},\mathcal{E}^t_i, X_i^t\}$ since we can assume each graph is built on the union set $\mathcal{V}$ and treat the absent nodes as isolated. $\mathcal{E}_i^t$ is the set of edges at time $t$ in $\mathcal{G}$. $X^t_i\in \mathbb{R}^{l\times n}$ represents the node attribute matrix, where $l$ is the dimension of the node attributes.
Let $A_i^t$ be the adjacency matrix of $G_i^t$.

In this paper, we focus on node-level tasks (e.g. classification and regression) for dynamic graphs.  
We consider the dataset as $\{(\mathcal{G}_i, {\bm y}_i)\}_{i=1}^N$, where $\mathcal{G}_i$ is the $i$-th dynamic graph, and ${\bm y}_i\in\mathbb{R}^{n}$ is the node-level labels assigned to the nodes in the last snapshot of dynamic graph $\mathcal{G}_i$.

\noindent\textbf{Permutation:} Here, we define a permutation function to simplify the notation in the rest of the paper. Let $\tau(t):=T- [(T-t+1) \mod T]$, where $T$ refers to the number of snapshots. This function maps $t\to t-1$ for $t\in [2,T]$ and maps $t\to T$ for $t=1$. 


\noindent\textbf{Implicit Models:} In general, implicit models~\cite{geng2021training, el2021implicit, gu2020implicit} have the form $ Z = f(Z,X)$,
where $f$ is a function that can be parameterized by a neural network, $X$ is the data and $Z$ is the learned representation. We can obtain the fixed-point representation via iteration $Z^*=\lim_{k\to \infty} Z_{k+1}=\lim_{k\to \infty} f(Z_{k},X) = f(Z^*,X)$. 

Thus, the key to designing an implicit model for dynamic graphs is to provide a function $f$ that leads to converged representations.





\subsection{Implicit Model for Dynamic Graphs}
 We first consider a single dynamic graph $\mathcal{G}=\{G^1,...,G^T\}$ with $T$ snapshots and discover its well-posedness condition in this section. We then generalize the well-posedness conclusion for a set of dynamic graphs later. All the proofs are provided in the Appendix. Now,
let us consider the following stacked GCN model:
\vspace{-0.01in}
\begin{align}
\label{sys_matrix}
    Z^1_{k+1} &= \sigma(W^1Z^{\tau(1)}_kA^1+VX^1)\nonumber\\
    Z^2_{k+1} &= \sigma(W^2Z^{\tau(2)}_kA^2+VX^2)\nonumber\\
    &\cdots \nonumber \\
    Z^T_{k+1} &= \sigma(W^TZ^{\tau(T)}_kA^T+VX^T)
\end{align}

In the model presented above, the embeddings $Z^2_{k+1}$ of the nodes in the second time stamp in the $(k+1)$-th layer depend on the embeddings $Z^1_k$ of nodes in the first time stamp learned in the $k$-th layer and the feature of the nodes in the second time stamp $X^2$. This design enables us to propagate information between time stamps when stacking layers.
The parameters for the $t$-th layer of the model are denoted as $W^t\in \mathbb{R}^{d\times d}$ and $V\in \mathbb{R}^{d\times l}$ with $V$ being a shared weight across all layers. Note that the proposed model and the corresponding theory still hold when $V$ is not shared. We opt for a shared $V$ for simplicity (thorough empirical discussion on this choice is presented in the Experiment section). Following the principle of the implicit model \citep{el2021implicit, bai2019deep, gu2020implicit}, we apply our model iteratively infinite times. If the process converges, we consider the converged result $\{Z^1_\infty, \dots, Z^T_\infty \}$ as the final embeddings. Consequently, the final embeddings have to satisfy the system of equations in (\ref{sys_matrix}) and can be considered a fixed-point solution to (\ref{sys_matrix}). However, at this point, it is not clear whether the fixed-point solution always exists for arbitrary graph $\GT$.

\textit{Well-posedness} is a property that an implicit function, such as in (\ref{sys_matrix}), possesses a unique fixed point solution. While Gu et al. \cite{gu2020implicit} demonstrated the well-posedness of a single-layer implicit GCN on a single static graph, the question remains open for dynamic graphs.
To establish the well-posedness property for our model, we first introduce its vectorized version as follows. 
\begin{align}
\label{sys}
    z^1_{k+1} &= \sigma(M^1z^{\tau(1)}_k+\textbf{vec}(VX^1))\nonumber\\
    z^2_{k+1} &= \sigma(M^2z^{\tau(2)}_k+\textbf{vec}(VX^2))\nonumber\\
    &\cdots \nonumber \\
    z^T_{k+1} &= \sigma(M^Tz^{\tau(T)}_k+\textbf{vec}(VX^T))
\end{align}
where $z=\textbf{vec}(Z)$ is column-wise vectorization of $Z$, and $M^i=(A^i)^\top\otimes W^i$ where $\otimes$ is the Kronecker product. 
Note that Equations (\ref{sys}) can also be expressed in a single matrix form. This transformation involves sequentially connecting the shared nodes between the graphs. Thus, the formula (\ref{sys}) can be reformulated as follows:

\begin{align}
\label{eq1}
    \begin{bmatrix}
    z^1\\
    z^2\\
    z^3\\
    \vdots\\
    z^T
    \end{bmatrix}=\sigma \left(\begin{bmatrix}
    0&0&\cdots &0&M^1\\
    M^2&0&\cdots &0&0\\
    0&M^3&\cdots &0&0\\
    \vdots &\vdots&\ddots &\vdots&\vdots\\
    0&0&\cdots &M^T &0
    \end{bmatrix}\begin{bmatrix}
    z^1\\
    z^2\\
    z^3\\
    \vdots\\
    z^T
    \end{bmatrix} + \begin{bmatrix}
       \textbf{vec}( VX^1) \\ \textbf{vec}( VX^2)\\ \textbf{vec}( VX^3) \\ \vdots \\ \textbf{vec}( VX^T)
    \end{bmatrix} \right)
\end{align}
Here, we omit the subscript for simplicity. Equation (\ref{eq1}) represents a single equilibrium form of Equation (\ref{sys}). It can also be viewed as the time-expanded static version~\cite{} of our original dynamic graph $\GT$. Based on the Banach fixed-point theorem \cite{royden1968real}, the Equation (\ref{eq1}) admits a unique fixed-point if the right-hand side is a contractive mapping w.r.t. $z$. Therefore, we express the well-posedness condition for our model as follows,

\begin{theorem}
For any element-wise non-expansive function $\sigma(\cdot)$, the coupled equilibrium equations in (\ref{sys}) have a unique fixed point solution if $\|\mathcal{M}\|_{op}< 1$, where $\mathcal{M}$ define as $$\begin{bmatrix}
    0&\cdots &0&M^1\\
    M^2&\cdots &0&0\\
    \vdots &\vdots&\ddots &\vdots\\
    0&\cdots &M^T &0
    \end{bmatrix}$$
and $\|\mathcal{M}\|_{op}$ is the operator norm of $\mathcal{M}$, which is the largest absolute eigenvalue. Furthermore, this is equivalent to $\|M^t\|_{op}< 1$ for any $t=1,...,T$.
\end{theorem}

In order to maintain $\|M^t\|_{op}< 1,\forall t=1,...,T$, it is necessary to ensure that the condition $\lambda_{\text{pr}}(|W^t|)\lambda_{\text{pr}}(A^t)< 1$ is satisfied, where $\lambda_{\text{pr}}(\cdot)$ represents the Perron-Frobenius eigenvalue. However, satisfying the condition is challenging in general as it is hard to track the eigenvalue as the underlying matrix changes.
To overcome this challenge, we impose a more stringent requirement on $W$ which is more easily enforceable by leveraging a convex projection. We formally state this in the following theorem.

\begin{restatable}{theorem}{thmtwo}
    Let $\sigma$ be an element-wise non-expansive function. If the coupled equilibrium equations satisfy the well-posedness condition, namely $\|\mathcal{M}^t\|_{op}\le\|W^t\|_{op}\|A^t\|_{op}<1, \forall t=1,..., T$, then there exists rescale coupled equilibrium equations, which satisfy the condition $\|W^t\|_{\infty}\|A^t\|_{op}<1, \forall t=1,...,T$, and the solutions of these two equations are equivalent.
\end{restatable}

As previously stated, the theorem is formulated for a single dynamic graph. In the following Remark, we present a broader and more general conclusion for a set of dynamic graphs.
\begin{remark}
    Considering a set of dynamic graphs denoted as $\{\mathcal{G}_i\}_{i=1}^N$, achieving fixed-point representations for all dynamic graphs using our model is guaranteed if the condition $\|W^t\|_{\infty}\|A^t_i\|_{op}<1$ holds for all time steps $t=1,...,T$ and all graph indices $i=1,...,N$.
\end{remark}

In practice, we can track the matrices with the largest operator norms ($\max_{i}\|A^t_i\|_{op}$) at each time step. Focusing on these "critical" matrices and their corresponding constraints ensures our model meets the required conditions.

\noindent\textbf{Incorporating Downstream Tasks: }
Based on the established conditions, we can obtain the fixed-point representation by iteratively applying our model. Now, we want the fixed-point representations suited for specific downstream tasks with their own optimization objective. Let us now introduce the comprehensive objective which incorporates both the application loss and the convergence requirements mentioned above. To this end, we utilize a neural network $f_\theta(.)$, parameterized by $\theta$, to map graph embeddings to their respective targets. Let $\mathcal{W}:= \{W^1,...,W^T\}$. The comprehensive objective can now be summarized as follows:

\medskip

\noindent\fbox{%
\parbox{0.9\columnwidth}{%
\begin{align}
\label{eq_z}
    \min_{\theta, \mathcal{W}, V}& \mathcal{L}(\theta, \mathcal{W}, V)=\sum_{i=1}^N\ell(f_\theta(z^T_{i}), \bm{y}_i)\\\nonumber
    \text{s.t.}~& z_{i}^1=\sigma \left(\left((A^1_{i})^\top\otimes W^1\right)z_{i}^{\tau(1)} +  \textbf{vec}(VX^1_{i}) \right)\\\nonumber
    & \cdots\\\nonumber
    & z_{i}^T=\sigma\left(\left((A^T_{i})^\top\otimes W^T\right)z_{i}^{\tau(T)} + \textbf{vec}(VX^T_{i})\right), \\\nonumber
    & \|W^t\|_\infty\le \frac{\kappa}{\|A^t_{i}\|_{op}}, \forall i=1,...,N,t=1,...,T
\end{align}

}
}
\medskip

Where $\ell$ is a loss function ( e.g. cross entropy loss, mean square error), and $\kappa$ is a positive number that is close to 1, which is added for numerical stability.

To find the fixed-point representation (i.e. forward pass), we can use fixed-point iteration or other root-finding algorithms. Backpropagation requires storing intermediate results, which is infeasible since there might be hundreds of iterations in discovering the representation. Thus, the key challenge in solving Equation (\ref{eq_z}) lies in determining how to perform backpropagation effectively, especially in the context of dynamic graphs.  



\section{Training}

In this section, we provide two algorithms to solve Equation (\ref{eq_z}). We first explore the stochastic gradient descent (SGD) method where we estimate the gradients leveraging the Implicit Function Theorem. This approach offers several advantages, such as eliminating the need to store intermediate results during the forward pass and enabling direct backpropagation through the equilibrium point. While widely used in various techniques \citep{bai2019deep, gu2020implicit}, this approach presents certain drawbacks when applied to our specific model, particularly regarding computational overhead. Subsequently, we introduce an efficient training algorithm for our model, which adopts a bilevel viewpoint of our problem. This novel approach allows us to overcome the limitations of the vanilla SGD method, resulting in improved computational efficiency during training.

\subsection{SGD with Implicit Differentiation} The algorithm operates as follows: it first finds the fixed-point embedding through iteration and then computes the gradient based on this embedding. It then updates the weights using SGD. The main obstacle in our way is estimating the gradient.

The parameters that need to be updated are $W,V$ (from GCN layers), and $\theta$ (from the classifier). The gradient with respect to parameter $\theta$ can be obtained as $\frac{\partial \mathcal{L}}{\partial \theta}$, which can be computed using autograd functions given the fixed point. However, computing the gradient for other parameters presents a greater challenge. Let $\frac{\partial \mathcal{L}}{\partial P_i}$ represent the gradient with respect to $W_i$ or $V_i$. Then, the gradient is computed as $
\frac{\partial \mathcal{L}}{\partial P_i}=\sum_{j=1}^T\frac{\partial\mathcal{L}}{\partial z_j}\frac{\partial z_j}{\partial P_i}$. The computation of $\frac{\partial\mathcal{L}}{\partial z_j}$ can be achieved through the autograd mechanism. However, determining $\frac{\partial z_j}{\partial P_i}$ is non-trivial due to the cyclic definition of $z_j$.

By the definition of our model, the learned embeddings must satisfy the following equations:
\begin{align}
    \label{eq:grad}
    \nonumber
    z^1-\sigma(M^1z^{\tau(1)}&+\textbf{vec}(V X^1))=0\\\nonumber
    z^2-\sigma(M^2z^{\tau(2)}&+\textbf{vec}(V X^{2}))=0\\\nonumber
    &\vdots\\
    z^T-\sigma(M^Tz^{\tau(T)}&+\textbf{vec}(V X^{T}))=0
\end{align}
We apply column-wise vectorization on matrices $W$ and $V$ respectively to obtain $w$ and $v$. We can then calculate the gradients $\frac{\partial z}{\partial w}$ and $\frac{\partial z}{\partial v}$ using the implicit function theorem. Here, we will show the details of the derivation. 

Let $Z^t_{ij}, A^t_{ij}, X^t_{ij}$ and $W^t_{ij}$ denote the element at the $i$-th row and $j$-th column of $Z^t, A^t, X^t$ and $W^t$, respectively. Taking the derivative of element-wise form of Equation (\ref{eq:grad}), $Z_{ij}^t-\sigma(\sum_l\sum_nW^t_{il}Z^{\tau(t)}_{ln}A^t_{nj}+\sum_lV_{il}X_{lj}^t)=0$, with respect to $W_{bc}^a$ we obtain
\begin{align*}
    \frac{\partial Z_{ij}^t}{\partial W_{bc}^a}-\Sigma^{t}_{ij}\left(  \sum_n\delta_{at}\delta_{bi}Z^{\tau(t)}_{cn}A_{nj}^t + \sum_l \sum_n W_{il}^t\frac{\partial Z_{ln}^{\tau(t)}}{\partial W_{bc}^a}A_{nj}^t \right)=0
\end{align*}
where $\delta_{at}$ is the indicator function which equals 1 only when $a=t$, and $\odot$ denotes element-wise multiplication. Let $\sigma^{\prime}(.)$ represent the derivative $\sigma(.)$, and we define $$ \Sigma^t_{ij}:= \sigma^{\prime}(\sum_l\sum_nW^t_{il}Z^{\tau(t)}_{ln}A^t_{nj}+\sum_lV_{il}X_{lj}^t)$$
which means $\Sigma^t$ is a matrix and has a shape like $Z^t$.
Let $H^t:=(Z^{\tau(t)}A^t)^T$, and $H^t_{\cdot c}$ denotes the $c$-th column of $H^t$. Let $e_b$ be a column vector with value 1 at $b$-th entry and 0 elsewhere. Enumerating $Z^t_{ij}$ for all $i,j$, we have 
\begin{align*}
    \frac{\partial z^t}{\partial W_{bc}^a}-\textbf{vec}(\Sigma^t)\odot\left(  \delta_{at}e_b\otimes H^t_{\cdot c} + M^t\frac{\partial z^{\tau(t)}}{\partial W_{bc}^a} \right)=0
\end{align*}
Therefore, for any time stamp $t$, the gradient of $z^t$ with respect to the $a$-th layer of GCN, $w^a$, can be expressed as:
\begin{align}
    \label{grad:w}
    \frac{\partial z^t}{\partial w^a}-\Xi^t\odot\left(  \delta_{at}H^t\otimes I + M^t\frac{\partial z^{\tau(t)}}{\partial w^a} \right)=0
\end{align}
Where $\Xi^t$ is a matrix that has identical column vectors, and each column vector is the $\textbf{vec}(\Sigma^t)$. Similarly, we can compute the gradient of $Z^t_{ij}$ w.r.t. $V_{bc}$.
\begin{align*}
    \frac{\partial Z_{ij}^t}{\partial V_{bc}}-\Sigma^t_{ij}\left(  \delta_{bi}X_{cj}^t + W_{il}^t\frac{\partial Z_{ln}^{\tau(t)}}{\partial W_{bc}^a}A_{nj}^t \right)=0
\end{align*}

Therefore,
\begin{align}
\small
    \label{grad:v}
    \frac{\partial z^t}{\partial v}- \Xi^t\odot\left( (X^t)^\top\otimes I+M^t \frac{\partial z^{\tau(t)}}{\partial v}\right)=0
\end{align}

While these equations provide a path for gradient computation, it is essential to note that all gradients are interconnected within a system of equations. The resolution of such a system entails a substantial computational overhead.

\textbf{Per-iteration Complexity of naive gradient descent: } Equations (\ref{grad:w}) and (\ref{grad:v}) reveal that a set of equilibrium equations determines the gradients. Consequently, to compute the gradients, we need to solve these equations using fixed-point iteration. Each layer necessitates one round of fixed-point iteration, and in total, including $V$, we need to perform fixed-point iteration $T+1$ times. The major computational overhead arises from the multiplication of $M$ with the derivatives, resulting in a complexity of $O((nd)^2d^2)$. Each fixed-point iteration involves $T$ instances of such computations. Consequently, the overall runtime for each update is $O(T^2n^2d^4)$. Although the adjacency matrix is sparse, it only reduces the complexity to $O(T^2nd^4)$. This limits the application of our model in large-scale dynamic graphs and hampers our ability to utilize large embeddings. 
 
\subsection{Efficient Update via Bilevel Optimization}
To address the previously mentioned challenges, we turn to Bilevel Optimization \cite{colson2007overview} as a potential solution.
We reformulate the problem presented in Equation (\ref{eq_z}) as the following standard bilevel optimization problem. 

\medskip

\noindent\fbox{%
\parbox{0.9\columnwidth}{%

\begin{align}
\label{bimodel}
\small
\min_{\theta, \mathcal{W}, V}& \mathcal{L}(\theta, \mathcal{W}, V)=\sum_{i=1}^N\ell(f_\theta(z^T_i, \bm{y}_i)\\\nonumber
    \text{s.t.}~ &  z^T_{i} = \arg\min_{z}\|z-\phi(z,\mathcal{W},V;\mathcal{G}_i)\|^2_2\\ \nonumber
    & \|W^t\|_\infty\le \frac{\kappa}{\|A^t_{i}\|_{op}}, \forall i=1,...,N,t=1,...,T
\end{align}
}}

\medskip

\noindent Where $\phi(z,\mathcal{W}, V;\mathcal{G}_i)=\sigma(M^T_{i}...\sigma(M^1_{i}z+\textbf{vec}(VX^1_{i}))... +\textbf{vec}(VX^T_{i}))$. The main differences between these problems lie in the constraints. Equation (\ref{bimodel}) introduces explicit constraints solely on the last snapshot, leading to a multi-block bilevel optimization problem. This type of problem has been investigated recently by \cite{qiu2022large} and \cite{ hu2022multi}. \cite{qiu2022large} focus on top-K NDCG optimization, formulating it as a compositional bilevel optimization with a multi-block structure. Their approach simplifies updates by sampling a single block batch in each iteration and only updating the sampled blocks. \cite{hu2022multi} employs a similar technique but addresses a broader range of multi-block min-max bilevel problems.

However, these state-of-the-art bilevel optimization algorithms are designed to address problems with strongly convex lower-level problems, which does not hold true for our problem. It is evident that our lower-level problems in \ref{bimodel} are generally nonconvex with respect to $z$ since they involve highly nonlinear neural networks. Additionally, these methods utilize stochastic gradient descent on the lower level in each iteration, which may lead to potential extra computation in estimating the gradient. Nevertheless, it is crucial to note that the optimal solution to our lower-level problem corresponds to the fixed point of Eq (\ref{sys}). Leveraging this insight, we employ a fixed-point iteration to update the lower-level solution. We propose a single loop algorithm (see Algorithm \ref{algo}) with fixed-point updates.

\begin{algorithm}[t]
\caption{Bilevel Optimization for IGDNN}
\label{algo}
\begin{algorithmic}[1]
\REQUIRE {$\mathcal{D}=\{(\GT_i,\bm{y}_i)\}_{i=1}^N, \eta_1, \eta_2, \gamma$}       
\ENSURE{$\omega, \theta$}
\STATE Randomly initialize $z_j^0$ and $v_j^0$ for $j=1,...,N$
\FOR{$k=0,1,...,K$}
\STATE Sample a batch data $\mathcal{B}$
\STATE $\hat{z}_j^{k+1}=\begin{cases} (I-\eta_1)\hat{z}_j^{k}+\eta_1 \phi(\hat{z}_j^{k},\omega^t;\GT_i) & j\in\mathcal{B}\\ \hat{z}_j^{t} & \text{o.w.} \end{cases}$
\STATE $\hat{v}_j^{k+1}=\begin{cases} (I-\eta_2\nabla_{zz}^2 g(\hat{z}_j^k,\omega^k))\hat{v}_j^k+\eta_2 \nabla_{z} \ell_j(\hat{z}_j^k,\omega^k) & j\in\mathcal{B}\\ \hat{v}_j^{t} & \text{o.w.} \end{cases}$
\STATE Update gradient estimator $$\Delta^{k+1}=\frac{1}{|\mathcal{B}|}\sum_{j\in\mathcal{B}}\left[ \nabla_\omega \ell_j(\hat{z}_j^k,\omega^k)-\nabla^2_{\omega z}g_j(\hat{z}_j^k,\omega^k)\hat{v}_j^k \right]$$
\STATE $m^{k+1} = (1-\gamma)m^k + \gamma \Delta^{k+1} $
\STATE $\omega^{k+1} = \Pi_\Omega \left( \omega^k - \eta_0 m^{k+1}\right)$
\ENDFOR
\end{algorithmic}
\end{algorithm}

To better illustrate our algorithm, let $\omega=\{\mathcal{W},V\}$, and $g_i(z,\omega)$ represents the $i$-th-block lower-level problem, defined as $\|z-\phi(z,\omega;\mathcal{G}_i)\|^2_2$, and let $\ell_i(z,\omega):=\ell(f_\theta(z), y_i)$.
The key to solving the bilevel optimization problem in (\ref{bimodel}) is to estimate the hypergradient $\nabla \mathcal{L}(\theta, \omega)$ and backpropagate.
The hypergradient of our objective in (\ref{bimodel}) with respect to $\omega$ as follows:
\begin{align*}
\footnotesize
    \nabla \mathcal{L}(\theta, \omega) = &\frac{1}{N}\sum_{i=1}^N\nabla \ell_i(z_{i}^T,\omega)\\
    &-\nabla_{\omega z}^2 g_i(z_{i}^T, \omega)\left[ \nabla_{zz}^2 g_i( z_{i}^T, \omega) \right]^{-1}\nabla_z \ell_i(z_{i}^T, \omega)
\end{align*}

\begin{table*}[!t]
    \centering
 
    \caption{Performance for classification task (ROCAUC) and regression task (MAPE (\%)). Performances on Brain10, England-COVID, PeMS04, and PeMS08 for baseline methods are taken from \cite{gao2022equivalence}. The best performance for each dataset is highlighted in bold, while the second-best performance is underlined.}
    \resizebox{\textwidth}{!}{
    \begin{tabular}{c|ccc|cccccc}
    \toprule
     & \multicolumn{3}{c|}{Classification} & \multicolumn{6}{c}{Regression} \\ 
     Model & Brain10&DBLP5&Reddit4 & \multicolumn{2}{c}{England-COVID} & \multicolumn{2}{c}{PeMS04} & \multicolumn{2}{c}{PeMS08} \\
     &  &  &  & Trans. & Induc. & Trans. & Induc. & Trans. & Induc. \\ \midrule\rowcolor{Gray}
    EvolveGCN-O & 0.58$\pm$0.10 & 0.639$\pm$0.207 & 0.513$\pm$0.008 & 4.07$\pm$0.73 & 3.88$\pm$0.47 & 3.20$\pm$0.25 & 2.61$\pm$0.42 & 2.65$\pm$0.12 & 2.40$\pm$0.27 \\
    EvolveGCN-H & 0.60$\pm$0.11 & 0.510$\pm$0.013 & 0.508$\pm$0.008 & 4.14$\pm$1.14 & 3.50$\pm$0.42 & 3.34$\pm$0.14 & 2.84$\pm$0.31 & 2.81$\pm$0.28 & 2.81$\pm$0.23 \\\rowcolor{Gray}
    GCN-GRU & 0.87$\pm$0.07 & 0.878$\pm$0.017 & 0.513$\pm$0.010 & 3.56$\pm$0.26 & \underline{2.97$\pm$0.34} & 1.60$\pm$0.14 & 1.28$\pm$0.04 & 1.40$\pm$0.26 & 1.07$\pm$0.03 \\
    DySAT-H & 0.77$\pm$0.07 & \textbf{0.917$\pm$0.007} & 0.508$\pm$0.003 & 3.67$\pm$0.15 & 3.32$\pm$0.76 & 1.86$\pm$0.08 & 1.58$\pm$0.08 & 1.49$\pm$0.08 & 1.34$\pm$0.03 \\\rowcolor{Gray}
    GCRN-M2 & 0.77$\pm$0.04 & 0.894$\pm$0.009 & \underline{0.546$\pm$0.020} & 3.85$\pm$0.39 & 3.37$\pm$0.27 & 1.70$\pm$0.20 & 1.20$\pm$0.06 & 1.30$\pm$0.17 & 1.07$\pm$0.03 \\
    DCRNN & 0.84$\pm$0.02 & 0.904$\pm$0.013 & 0.535$\pm$0.007 & 3.58$\pm$0.53 & 3.09$\pm$0.24 & 1.67$\pm$0.19 & 1.27$\pm$0.06 & 1.32$\pm$0.19 & 1.07$\pm$0.03 \\\rowcolor{Gray}
    TGAT & 0.80$\pm$0.03 & 0.895$\pm$0.003 & 0.510$\pm$0.011 & 5.44$\pm$0.46 & 5.13$\pm$0.26 & 3.11$\pm$0.50 & 2.25$\pm$0.27 & 2.66$\pm$0.27 & 2.34$\pm$0.19 \\
    TGN & \underline{0.91$\pm$0.03} & 0.887$\pm$0.004 & 0.521$\pm$0.010 & 4.15$\pm$0.81 & 3.17$\pm$0.23 & 1.79$\pm$0.21 & 1.19$\pm$0.07 & 1.49$\pm$0.26 & 0.99$\pm$0.06 \\\rowcolor{Gray}
    GRU-GCN & \underline{0.91$\pm$0.03} & 0.906$\pm$0.008 & 0.525$\pm$0.006 & \underline{3.41$\pm$0.28} & \textbf{2.87$\pm$0.19} & \underline{1.61$\pm$0.35} & \underline{1.13$\pm$0.05} & \underline{1.27$\pm$0.21} & \underline{0.89$\pm$0.07} \\
    IDGNN & \textbf{0.94$\pm$0.01}& \underline{0.907$\pm$0.005} & \textbf{0.556$\pm$0.017} & \textbf{2.65$\pm$0.25} & 3.05$\pm$0.25 & \textbf{0.53$\pm$0.05} & \textbf{0.63$\pm$0.04} & \textbf{0.45$\pm$0.11} & \textbf{0.50$\pm$0.05} \\ \bottomrule
    \end{tabular}}
    \label{tab:exp_combined}
\end{table*}


 Explicitly computing the inverted  Hessian matrix $\left[ \nabla_{zz}^2 g_i( z_{i}^T, \omega) \right]^{-1}$ in computation of our hypergradient is computationally expensive. 
Inspired by \cite{hu2022multi} and \cite{qiu2022large}, we instead directly approximate $\left[ \nabla_{zz}^2 g_i( z_{i}^T, \omega) \right]^{-1}\nabla_z \ell_i(z_{i}^T, \omega)$ using $v_i$ for each block by moving average estimation (line 5).  More specifically, we track the optimal point of $\min_v \frac{1}{2}v^\top\nabla_{zz}^2 g_i( z_{i}^T, \omega) v - v^\top\nabla_z \ell_i(z_{i}^T, \omega)$ for each block by maintaining $v_i$. Let $\hat{z}_i$ be a moving average approximation to the optimal lower-level solution $z_{i}^T$. We use a single fixed-point update for $\hat{z}_i$ (line 4). We do not want to update all blocks in every iteration since this is impractical when the number of blocks is large. To address this issue, we use stochastic training. For sampled blocks, we update their $\hat{z}$ and $\hat{v}$, and we compute the hypergradient (line 6).  Note that the multiplication $\nabla^2_{\omega z}g_j(\hat{z}_j^k,\omega^k)\hat{v}_j^k$ (in line 6) can also be efficiently computed using Hessian vector product~\cite{pearlmutter1994fast}. .As a result, the training time for our algorithm is proportional to normal backpropagation, eliminating the need for fixed-point iterations.

Note that in cases where the lower-level problem is strongly convex, the errors introduced by these approximations are well-contained \citep{hu2022multi}. Our lower-level problem admits a unique fixed point, hence, employing fixed-point iteration becomes an efficient means of attaining the optimal lower-level solution, akin to the effectiveness of gradient descent under strong convexity.
Therefore, it is justifiable to assert that our approximations are effective in this scenario, with empirical evidence robustly endorsing their practical efficacy.

\textbf{Per-iteration Complexity of bilevel optimization: } The main computational overheads are updating $v$ and estimating gradient. Both steps are involved with estimating a huge Hassian matrix, but, in practice, we use Hessian vector product to avoid explicitly computing the Hessian matrix. Therefore, the dominant runtime of bilevel optimization is three times backpropagation. Each backpropagation takes $O(Tnd^2+Tn^2d)$.

\section{Experiments}

In this section, we present the performance of IDGNN in various tasks, including effectiveness in capturing long-range dependencies and avoiding oversmoothing on a synthetic dataset. We benchmark IDGNN against nine state-of-the-art baselines on multiple real-world datasets, comprising three node classification and four node regression datasets. Key dataset statistics are summarized in Table \ref{data}, with further details available in section \ref{subsection:datasets}. Due to the space constraints, specifics on experimental setup and hyperparameters are provided in Appendices \ref{appendix:setup} and \ref{appendix:hyper}, respectively. 
Our code is available for reproducibility. \footnote{\url{https://github.com/yongjian16/IDGNN}}
.
 

\begin{table}[h]
\small
    \centering
    \rowcolors{2}{white}{Gray} 
    \caption{Statistics of datasets. $N$: number of dynamic graphs, $|V|$: number of nodes, $\min |E|$: minimum number of edges, $\max |E|$: maximum number of edges, $T$: window size, $d$: feature dimension, $y$ label dimension}
    \begin{tabular}{cccccccc}
    \toprule
        & $N$& $|V|$& $\min|E|$& $\max|E|$ & $T$ &$d$ & $y$\\ \midrule
        Brain10 & 1 & 5000 & 154094 & 167944 & 12 & 20 & 10  \\ 
        DBLP5   & 1 & 6606 & 2912 & 5002 & 10 & 100 & 5 \\
        Reddit4 & 1 & 8291 & 12886 & 56098 & 10 & 20 & 4 \\
        PeMS04& 16980&307&680&680&12&5&3 \\
        PeMS08 & 17844& 170&548&548&12&5&3 \\ 
        English-COVID & 54 & 129 & 836&2158&7&1&1 \\ \bottomrule
    \end{tabular}
    
    \label{data}
\end{table}

\begin{figure*}[h]
     \centering
     \includegraphics[width=\textwidth]{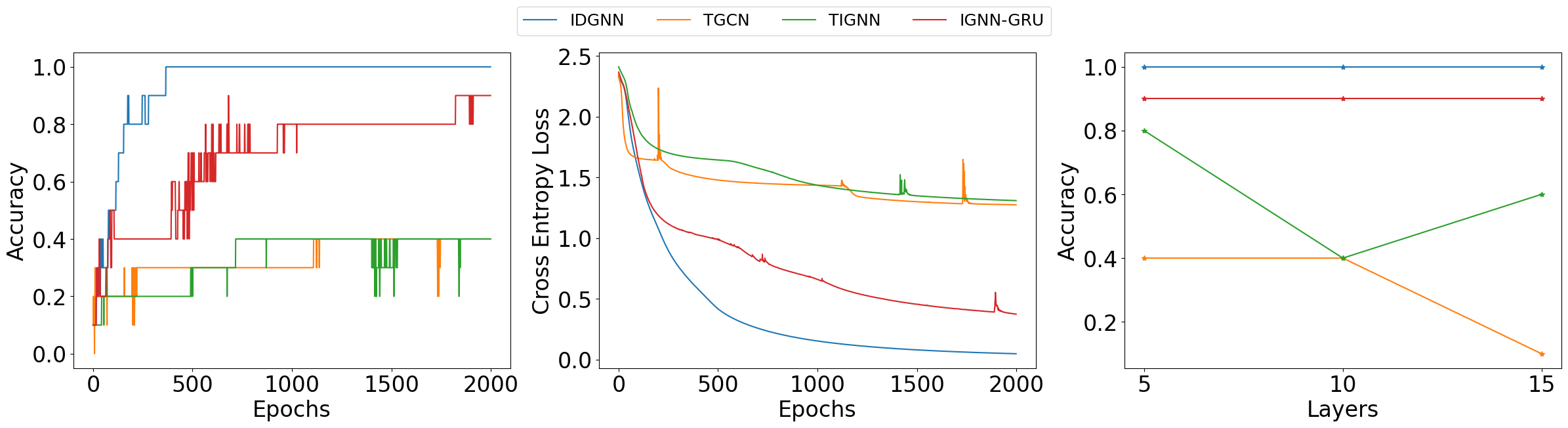}
     \caption{The left and middle are accuracy and loss curves when using 10 layers. The x-axis is epochs, and the y-axis is accuracy and cross entropy loss, respectively. The right plot represents the accuracy results of all baselines under different layer settings.}
        \label{2graphs}
\end{figure*}

\subsection{Datasets}\label{subsection:datasets}

\textbf{Brain} dataset is derived from a real-world fMRI brain scans dataset \footnote{https://tinyurl.com/y4hhw8ro}. In this dataset, nodes represent brain tissues, and edges capture nodes' activation time similarity in the examined period. The node attributes are generated from the fMRI signals \cite{Xu2019AdaptiveNN}. Given the temporal graph, our goal is to predict the functionality of each node, which is one of the $\textbf{10}$ categories.

\noindent\textbf{DBLP} is an extracted co-author network from DBLP website \footnote{https://dblp.org/}, where nodes are authors and edges represent co-authorship relationships. The extracted authors are from $\textbf{5}$ research areas \cite{Xu2019AdaptiveNN}, serving as our class labels. Note attributes are word2vec representations of titles and abstracts from related papers.

\noindent\textbf{Reddit} dataset is extracted from a popular online platform Reddit \footnote{https://www.reddit.com/} \cite{hamilton2018inductive}, where nodes represent posts, and edges are constructed between nodes having similar keywords. The node attributes are word2vec representations of the comments of a post \cite{Xu2019AdaptiveNN}, and node labels correspond to one of the $\textbf{4}$ communities or ``subreddits" to which the post belongs.

\noindent\textbf{PeMS04 \& PeMS08}
These two datasets represent traffic sensor networks in two distinct districts of California during various months. In this context, each node symbolizes a road sensor, while the edge weights denote the geographical distance between two sensors (with uniform weights across all snapshots). Every sensor captures average traffic statistics such as flow, occupancy, and speed over a 5-minute interval. The datasets utilized in this study are identical to those in \cite{gao2022equivalence}, where we exclusively forecast attribute values for a single future snapshot.

\noindent\textbf{England-COVID} This temporal graph dataset is taken from \cite{gao2022equivalence}, which is created by aggregating information from the mobility log released by Facebook under Data For Good program \footnote{https://dataforgood.fb.com/tools/disease-prevention-maps/} for England \cite{panagopoulos2020transfer}. The nodes are cities, and the edges are transportation statistics between them. Given a temporal graph of length $7$, which represents the past week, we need to predict the infection rate of the next day.

\subsection{Regression}

For node-level tasks, there are two evaluation paradigms: transductive and inductive. Transductive evaluation allows the model to access the nodes and edges in testing data during training (i.e., only the labels are hidden), while inductive evaluation involves testing the model on new nodes and edges. In simpler terms, transductive evaluation separates training and testing sets based on nodes, while inductive evaluation separates them based on time stamps (i.e., models are trained on past time stamps and tested on the future).
Here, we conduct experiments on both settings.

The datasets we used for the regression task are England-COVID, PeMS04, and PeMS08. We use the mean average percentage error (MAPE) as our evaluation metric. The results are reported in Table \ref{tab:exp_combined} with mean MAPE and standard deviation. 
Our proposed method outperforms other methods in both transductive and inductive settings, with the exception of the inductive case in England-COVID. Our method demonstrates a significant improvement for PeMS04 and PeMS08, particularly in the transductive learning scenario. In comparison to the second-best method, our proposed model reduces the error by over 1\% of MAPE, but our model on the inductive learning scenario does not enjoy such improvement.  We attribute this difference to our model's tendency to separate nodes, even when they have the same labels and topology. We delve into this phenomenon in the Appendix \ref{appendix:emb}. 

\subsection{Classification}

We conducted classification experiments on Brain10, DBLP5, and Reddit4 datasets. Since these datasets consist of only one dynamic graph, we focused on testing the transductive case. The evaluation was done using the Area under the ROC Curve (AUC) metric. The average prediction AUC values and their corresponding standard deviations are presented in Table \ref{tab:exp_combined}. Our proposed model achieved the top rank in 2 out of 3 datasets and was the second best in the remaining dataset. These results demonstrate that our model successfully captures the long-range dependencies within the dynamic graphs, as reflected in the learned embeddings.



\begin{table}[t]
    \centering
    \caption{Memory and runtime comparison results for all methods on reddit4 and DBLP5 datasets. We report the memory usage using MB and runtime using seconds per window}.
    \begin{tabular}{ccccc}
    \toprule 
    & \multicolumn{2}{c}{Reddit4}& \multicolumn{2}{c}{DBLP5}\\
    & Mem. & Time & Mem. & Time \\
    \midrule\rowcolor{Gray}
        EvolveGCN-O	&\textbf{42}	&0.0649±0.017 &\textbf{52}&	0.0672±0.014\\ 
        EvolveGCN-H	&\underline{52}&	0.0904±0.020 &	82&	0.0997±0.037\\\rowcolor{Gray}
        GCN-GRU	&221	&0.0733±0.012&	200&	0.1142±0.045\\
        DySAT-H	&181	&0.1613±0.056&	165&	0.1343±0.012\\\rowcolor{Gray}
        GCRN-M2	&322	&0.4345±0.080&	319&	0.4934±0.076\\
        DCRNN	&223	&0.1697±0.019&	278&	0.2121±0.039\\\rowcolor{Gray}
        TGAT	&793	&0.0750±0.014&	338&	0.0770±0.015\\
        TGN	&450	&0.0417±0.004&	233&	0.0454±0.012\\\rowcolor{Gray}
        GRU-GCN	&4116	&\textbf{0.0199±0.008}&	580&	\textbf{0.0161±0.007}\\
        IDGNN	&89	&\underline{0.0291±0.007}&	\underline{75}&	\underline{0.0302±0.002}\\\bottomrule
    \end{tabular}
    \label{mandt2}
\end{table}

\subsection{Long-range Dependency and Oversmoothing}
\label{longrange}
\noindent \textbf{Long-range Dependency:}
The toy example aims to test the ability of all approaches to capture long-range dependencies. The toy data we constructed consists of $\{5,10,15\}$ snapshots, with each snapshot being a clique of $10$ nodes. Each node has $10$ associated attributes. The task is to classify nodes at the last snapshot, where each node represents its own class (i.e., there are a total of 10 classes). The node attributes consist of randomly generated numbers, except for the first snapshot, which uses the one-hot representation of the class. Successful classification of this dataset requires effective information aggregation starting in the initial time stamp, propagating the class label information over time, and avoiding oversmoothing as the information is propagated. In this dataset, there are no testing nodes; all nodes are used for training. 

The training results are presented on Figure \ref{2graphs}. Our model is compared with TGCN~\citep{zhao2019t}. We also propose two more modified baselines: IGNN-GRU and TIGNN, which are obtained by replacing the GCNs within GCN-GRU~\citep{seo2018structured} and TGCN by IGNN.  
We ensure the comparison is fair by ensuring a similar number of parameters are used, and we test all models on $\{5,10,15\}$ layers. All methods are trained for a maximum of 2000 epochs, followed by the hyper-parameter selection approach described in the Appendix. As shown in the figure, we explored the impact of varying the number of layers on both baseline models and our proposed model. We documented the resulting accuracies accordingly. Notably, TGCN exhibits a discernible pattern of over-smoothing, evidenced by a performance decline with increasing layers.
Conversely, IGNN-GRU and TIGNN do not demonstrate such susceptibility. Additionally, we observed challenges in data fitting for both methods, whereas our model consistently achieved 100\% accuracy across different layer configurations. This experiment underscores the robustness of our architecture in fully unleashing the potential of implicit models in dynamic graph learning.

Moreover, we perform an alternative experiment, wherein we shift label information from snapshot $1$ to snapshot $5$. This adjustment ensures consistent difficulty in leveraging label information across all models. Due to space constraints, we provide the detailed results in the Appendix; however, the overall conclusion remains unaffected.

\noindent \textbf{Oversmoothing:}
In this part, we employ Dirichlet energy to assess the smoothness of the learned embeddings \cite{rusch2023survey}. The Dirichlet energy measures the mean distance between nodes and is defined as follows:
\vspace{-0.05in}
$$ DE(Z) = \sqrt{\frac{1}{|\mathcal{V}|}\sum_{i\in \mathcal{V}}\sum_{j\in \mathcal{N}_i}\|Z_i-Z_j\|^2}$$
Here, $\mathcal{N}_i$ denotes the neighbors of node $i$. Low Dirichlet energy means smooth embedding. We compare our model with two different baselines:
\begin{itemize}
    \item Static: We create a static graph by averaging the adjacency matrices of all snapshots and using the feature matrix of the last snapshot as its feature. Subsequently, we learn the embeddings using a vanilla multi-layer GCN.
    \item W/o loop: This model shares the structure of IDGNN but does not enforce the learned embeddings to be a fixed point.
\end{itemize}
We evaluate these methods on the Brain10 dataset, and the results are presented in Table \ref{oversmooth}. The "static" method exhibits oversmoothing as the layer count increases, and a similar issue is observed for the "W/o loop" method, despite stacking one layer for each snapshot. These findings strongly indicate that our model effectively addresses oversmoothing in dynamic graph learning by leveraging the implicit model structure.
\begin{table}[t]
    \centering
    \rowcolors{2}{white}{Gray} 
    \caption{Evaluating the smoothness of embeddings and the AUC on Brain10. }
    \begin{tabular}{ccc}
    \toprule
        & Dirichlet Energy $\uparrow$ & AUC $\uparrow$\\ \midrule
        Static (3 layers) & 10.660 & 89.79 \\
        Static (4 layers) & 8.781 & 88.67 \\
        Static (5 layers) & 3.399 & 81.49 \\
        W/o looping & 8.957 & 85.61 \\
        IDGNN & \textbf{35.299} & \textbf{94.87} \\
        \bottomrule
    \end{tabular}
    \label{oversmooth}
\end{table}

\begin{table}[t]
    \centering
    
    \caption{Runtime and performance comparison between SGD and Bilevel Methods.}
    \begin{tabular}{ccccc}
  \toprule
        & \multicolumn{2}{c}{Runtime (s/win)} & \multicolumn{2}{c}{Performance} \\
        & SGD& Bilevel & SGD& Bilevel
        \\\midrule\rowcolor{Gray}
        Brain10 & 624& \textbf{0.390}&94.7& \textbf{94.9}\\
        PeMS04 &0.72 & \textbf{0.049}& \textbf{0.63} & \textbf{0.63}\\\rowcolor{Gray}
        PeMS08 &0.29 & \textbf{0.046}& 0.56 & \textbf{0.50} \\
        England-COVID &0.092 &\textbf{0.030} &\textbf{2.97} &3.05\\\bottomrule
     \end{tabular}
        \label{tab:week2}
\end{table}

\subsection{Efficiency}
We compare runtime and performance between SGD and bilevel optimization algorithms. To this end, we compare Brain10, England-COVID, PeMS04, and PeMS08. The results are summarized on the Table \ref{tab:week2}. The results are computed by averaging the runtime of a whole epoch with the number of dynamic graphs $N$. 



These methods have similar performance, but the runtime results show that the bilevel optimization algorithm is much faster than the SGD. Especially, in the Brain10 dataset, bilevel algorithm achieves \textit{1600}
times of speedup compared with SGD. Furthermore, we notice that the ratio of runtimes in PeMS04 and PeMS08 is $\frac{0.72}{0.29}=2.48$, and the squared ratio of their number of nodes is $(\frac{307}{170})^2=3.26$. This confirms our complexity result for SGD, which is quadratic regarding the number of nodes. On the other hand, the bilevel method exhibits only linear dependency. We also present the memory usage and runtimes of all methods on Reddit4 and DBLP5.
The memory efficiency of implicit models comes from the fact that implicit models can use few parameters and do not need to store the intermediate results. However, we need to store intermediate results and backpropagate for our bi-level method. Due to the simple RNN-free architecture of our method, our approach is competitive in runtime and memory. We provide a memory and runtime comparison on DBLP5 and Reddit4. The results are summarized in Table \ref{mandt2}.

\subsection{Complexity}
To contrast with the emprical runtime, here we present the theoretical time complexities for our approach and the baselines. For an arbitrary temporal graph, we denote the total number of snapshots as $T$, the total number of nodes as $n$, the number of edges of time $t$ as $E_t$, and $E_{agg}$ as the number of aggregated edges, which satisfies $E_{agg}\ll\sum_t E_t$ or $ E_{agg}\approx\sum_t E_t $. Some basic complexities: GRU, self-attention, and LSTM have complexity $O(T^2)$ if the input sequence has length $T$; GCN and SpectralGCN have complexity $O(nd^2+Ed)$; GAT has complexity $O(nd+Ed^2)$. The complexities of all models are summarized in Table \ref{tab:comp}. Based on the complexity results, IDGNN is faster than EvolveGCN-O, EvolveGCN-H, GCN-GRU, GCRN-M2, and DCRNN due to the absence of RNN in our model. 
\begin{table}[]
    \centering
    \begin{tabular}{c|c}
\hline
    EvolveGCN-O&	$O(Td^2+Tnd^2+\sum_t E_td) $\\
EvolveGCN-H	&$O(Td^2+Tnd^2+\sum_t E_td) $\\
GCN-GRU	&$O(Td^2+Tnd^2+\sum_t E_td) $\\
DySAT	&$O(Td^2+Tnd+\sum_t E_td^2) $\\
GCRN-M2&	$O(Td^2+Tnd^2+\sum_t E_td) $\\
DCRNN	&$O(Td^2+Tnd^2+\sum_t E_td) $\\
TGAT	&$O(Tnd+\sum_t E_td^2) $\\
TGN	&$O(E_{agg}d+Tnd^2+\sum_t E_td^2) $\\
GRU-GCN	&$O(E_{agg}d+Tnd^2+ TE_{agg}d^2) $\\
IDGNN	&$O(Tnd^2+\sum_t E_td) $ \\\hline
\end{tabular}
    \caption{Summary of complexities of all models}
    \label{tab:comp}
\end{table}

\begin{table}[b]
    \centering
    \caption{Evaluating different variants of our model. Presenting AUC results on Brain10 datasets }
    
    \begin{tabular}{cccccc}
    \toprule
        & Share both & IDGNN & Share $V$& Not share & W/o loop\\ \midrule
        AUC & 94.29 & 94.87 & 94.87 & 94.90 & 85.62\\\bottomrule
    \end{tabular}
    \label{Configu}
\end{table}

\subsection{Ablation Study}
In this section, we will delve into the various configurations of our model. Drawing from the properties mentioned earlier, our model can be represented by the formula (\ref{sys_matrix}). We have deliberately decided to assign different weights ($W$) to each timestamp while maintaining weight-tied for $V$. Alternatively, the model comprises $T$ layers of GCN and one linear layer. The linear layer serves to aggregate static information, while the GCNs handle dynamic information. To validate our architecture choice, we conducted a thorough comparison of our model against other configurations:
\begin{itemize}
    \item Share both: Both $W$ and $V$ are shared across layers.
    \item Share $V$: Only $V$ is shared across layers.
    \item Not share: $W$ and $V$ are not shared.
    \item W/o loop: as defined in Section \ref{longrange}.
\end{itemize}

The results are presented in Table \ref{Configu}. As observed, the share-both model exhibits the poorest performance. We believe this is due to the limited number of free variables available for learning, which makes the training process challenging. In our approach and the share-$V$ method, we achieve very similar results. Our model utilizes $T$ layers of GCN for dynamic information and one linear layer for static information, while the share-$V$ method employs one GCN and $T$ linear layers. The not-share method achieves the best result, although the improvement is negligible. However, it increases the parameter size, resulting in significant computational overhead. Hence, we opt for the current configuration, as it delivers good performance while minimizing parameter size, especially since the number of attributes may exceed the hidden dimension. We additionally evaluate an alternative baseline, denoted as "w/o loop," wherein we eliminate the looping structure from IDGNN. The obtained results reveal that this model exhibits the lowest performance, underscoring the efficacy of our proposed approach.

\section{Conclusions}
In this paper, we propose a novel implicit graph neural network for dynamic graphs. As far as we know, this is the first implicit model on dynamic graphs. We demonstrate that the implicit model we proposed has the well-posedness characteristic. We proposed a standard optimization algorithm using the Implicit Function Theorem. However, the optimization was too computationally expensive for our model. Hence, we proposed a novel bilevel optimization algorithm to train our proposed model. We conducted extensive experiments on 6 real-world datasets and one toy dataset.  The regression and classification tasks show that the proposed approach outperforms all the baselines in most settings. Finally, we also demonstrated that the proposed bilevel optimization algorithm obtains significant speedup over standard optimization while maintaining the same performance. A key limitation of our proposed approach is that it can only predict the consecutive snapshot. In the future, we plan on addressing this issue and also provide a diffusion model-based training algorithm.


\section{Acknowledgement}
We are grateful to the anonymous reviewers for their constructive comments and suggestions. This work was supported in part by the NSF Cybertraining 2320980, SCH 2306331, CAREER  1844403 and the CDC MInD Healthcare group under cooperative agreement U01-CK000594.

\clearpage
\bibliographystyle{ACM-Reference-Format}
\bibliography{ref}

\clearpage
\appendix

\appendix
\section{Experiment}

\subsection{Experiment Setup}
\label{appendix:setup}
We follow the procedure from \cite{gao2022equivalence} and utilize the provided code 
as the code base to compare all the baselines and our method. Specifically, we first split the dataset into three portions with ratios 70\%-10\%-20\% for training, validation, and testing, respectively. Splitting is based on nodes for transductive tasks and time for inductive tasks. We then normalize node attributes and edge attributes with the 0-1 normalization method. We train on the training portion, find the best hyperparameters using the validation set, and report the performance on the test set. We also use ROCAUC score to evaluate classification tasks and mean average percentage error (MAPE) for regression tasks.


\subsection{Hyperparameter}
\label{appendix:hyper}
For detailed baselines' architecture, please refer to \cite{gao2022equivalence}. Notice that, for all the methods and all task, we fixed the embedding size as $16$, and we searched the learning rates from $0.1$, $0.01$, and $0.001$. For Brain10, we observed that our method converged slowly, then we let it run for $1000$ epochs. For DBLP5 and Reddit4, we let all the methods run $500$ epochs for $5$ times for each learning rate and report the performance on the test set where the performance of the validation set is the best. 
For regression datasets, we run 100 epochs for England-COVID and 10 epochs for PeMS04/PeMS08.
The hyperparameter our model $\eta_1\in\{0.5,0.7,0.9,1\}, \eta_2\in\{0.001,0.01,0.1\}$.

\section*{Proofs}
\begin{lemma}
\label{a1}
    If $\|.\|_{op}$ is the matrix operator norm on $\mathbb{R}^{n\times n}$ and $\lambda_{\text{PF}}(A)$ outputs the largest absolute eigenvalue of $A$, then, for any matrix $A\in \mathbb{R}^{n\times n}$,
    $$\lambda_{\text{PF}}(A)\le \|A\|_{op}$$
\end{lemma}
\begin{proof}
    Let $\lambda$ be the eigenvalue of $A$, and let $x\neq 0$ be the corresponding eigenvector. From $Ax=\lambda x$, we have 
    $$AX=\lambda X$$
    where each column in $X$ is $x$. Further, we have
    $$|\lambda|\|X\|_{op}=\|\lambda X\|_{op}=\|AX\|_{op}\le \|A\|_{op} \|X\|_{op}$$
    Since $\|X\|>0$, taking the maximal $\lambda$ gives the result.
\end{proof}

\begin{restatable}{lemma}{lemmaone}
\label{lemma_1} 
    The equilibrium equation $z=\sigma(Mz+b)$ has a unique fixed point solution if $\||M|\|_{op}< 1$, where $\|.\|_{op}$ is the operator norm, and $\sigma(\cdot)$ is an element-wise non-expansive function.
\end{restatable}
\begin{proof}
    Based on Lemma \ref{a1} and Theorem 4.1 in \cite{gu2020implicit}, this lemma is immediately followed.
\end{proof}

\subsection*{Proof of Theorem 3.1}

\begin{proof}
By Lemma. \ref{lemma_1}, the well-posedness requirement for Formula. (\ref{eq1}) is $\|\mathcal{M}\|_{op}|\le 1$. Since Formula. (\ref{eq1}) and (\ref{sys}) are equivalent, the well-posedness requirement for Formula. (\ref{sys}) is also $\||\mathcal{M}|\|_{op}< 1$.

Let $M:=\begin{bmatrix} 0&M_1\\ \hat{M}&0 \end{bmatrix}$ where $\hat{M}:=\begin{bmatrix} M_2 &\hdots &\\
          0&\ddots&0\\
          0&\hdots&M_t
\end{bmatrix}$. Let $\Tilde{M}:=\begin{bmatrix} \hat{M}&0\\ 0&M_1 \end{bmatrix}$. Then 
\begin{align*}
    \||M|\|_{op}=\||\Tilde{M}|\begin{bmatrix} 0&I_m\\ I_n&0 \end{bmatrix}\|_{op}\le \||\Tilde{M}|\|_{op} \cdot\|\begin{bmatrix} 0&I_m\\ I_n&0 \end{bmatrix}\|_{op}\\
    =\||\Tilde{M}|\|_{op}=\max\{\||M_1|\|_{op},...,\||M_t|\|_{op}\}
\end{align*}

This means if all subsystems satisfy the largest eigenvalue constraint,  then the coupled equilibrium equation has a fixed point solution by Lemma \ref{lemma_1}.
\end{proof}

\subsection*{Proof of Theorem 3.2}
\thmtwo*
\begin{proof}
    Suppose $\{W^t\}$ satisfy $\|W^t\|_{op}\|A^t\|_{op}<1$ for all $t$, then the following equation has a unique fixed point.
    \begin{align*}
    \begin{bmatrix}
    z^1\\
    z^2\\
    z^3\\
    \vdots\\
    z^T
    \end{bmatrix}=\sigma \left(\begin{bmatrix}
    0&0&\cdots &0&M^1\\
    M^2&0&\cdots &0&0\\
    0&M^3&\cdots &0&0\\
    \vdots &\vdots&\ddots &\vdots&\vdots\\
    0&0&\cdots &M^T &0
    \end{bmatrix}\begin{bmatrix}
    z^1\\
    z^2\\
    z^3\\
    \vdots\\
    z^T
    \end{bmatrix} + \begin{bmatrix}
       \textbf{vec}( VX^1) \\ \textbf{vec}( VX^2)\\ \textbf{vec}( VX^3) \\ \vdots \\ \textbf{vec}( VX^T)
    \end{bmatrix} \right)
\end{align*}
and this condition implies $\|\mathcal{M}\|_{op}\le 1$. Based on Theorem 4.3, there exists a set of diagonal matrices $\{S^t\}$ such that $$\hat{W}^t=S^tW^t(S^t)^{-1}, \hat{V}=S^tV $$ Then the fixed-point of following equations
\begin{align*}
    \begin{bmatrix}
    \hat{z}^1\\
    \hat{z}^2\\
    \hat{z}^3\\
    \vdots\\
    \hat{z}^T
    \end{bmatrix}=\sigma \left(\begin{bmatrix}
    0&0&\cdots &0&\hat{M}^1\\
    \hat{M}^2&0&\cdots &0&0\\
    0&\hat{M}^3&\cdots &0&0\\
    \vdots &\vdots&\ddots &\vdots&\vdots\\
    0&0&\cdots &\hat{M}^T &0
    \end{bmatrix}\begin{bmatrix}
    \hat{z}^1\\
    \hat{z}^2\\
    \hat{z}^3\\
    \vdots\\
    \hat{z}^T
    \end{bmatrix} + \begin{bmatrix}
       \textbf{vec}( \hat{V}X^1) \\ \textbf{vec}( \hat{V}X^2)\\ \textbf{vec}( \hat{V}X^3) \\ \vdots \\ \textbf{vec}( \hat{V}X^T)
    \end{bmatrix} \right)
\end{align*}
satisfies the following relation
\begin{align*}
\begin{bmatrix}
    \hat{z}^1\\
    \vdots\\
    \hat{z}^T
    \end{bmatrix}=\begin{bmatrix}
    (S^1)^{-1}&\cdots &0\\
    \vdots &\ddots &\vdots\\
    0&\cdots &(S^T)^{-1} 
    \end{bmatrix}
    \begin{bmatrix}
    \hat{z}^1\\
    \vdots\\
    \hat{z}^T
    \end{bmatrix}
\end{align*}
\end{proof}

\subsection*{Proof of Remark 1}
\begin{proof}
    Given a set of dynamic graphs $\{\mathcal{G}_i\}_{i=1}^N$, we can construct a single dynamic graph by merging the snapshots that are from the same time stamp, then we obtain \begin{align}
        \hat{\mathcal{G}}=\{[G_i^1,...,G_N^1],...,[G_i^T,...,G_N^T]\}
    \end{align}
    Let $\hat{A}^t$ denote the adjacency matrix of $[G_i^t,...,G_N^t]$. By theorem 3.2, we need to ensure $\|W^t\|_{\infty}\|\hat{A}^t\|_{op}<1$. Since $\hat{A}^t$ contains $N$ disconnected graphs, $\|\hat{A}^t\|_{op}\le \max_{i}\|A^t_i\|_{op}$, which means $\|W^t\|_{\infty}\|A_i^t\|_{op}<1$ needed to be satisfied for all $i$. Since $t$ is arbitrary, the remark holds.
\end{proof}

\section{Embedding Visualization}
\label{appendix:emb}
In this section, we explore an interesting aspect of our method that can provide empirical insights into its ability to mitigate oversmoothing. We conduct experiments on a synthetic dataset that bears resemblance to toy datasets.

The dataset comprises 10 snapshots, with each snapshot representing a clique of 10 nodes. Each node is associated with 10 attributes. The nodes fall into two distinct classes, but we deliberately conceal the label information in the attributes of the first snapshot. Specifically, the first two dimensions of the attributes represent the one-hot encoding of the labels, while the remaining dimensions are set to zero. Additionally, we assign unique IDs to the nodes in sequential order. Nodes with IDs ranging from 0 to 4 belong to class 0, while those with IDs ranging from 5 to 9 belong to class 1. To assess the effectiveness of our method, we visually compare the embedding results with those obtained from TGCN.

Upon examining the visualizations, we observe that our model's embeddings exhibit gradual changes, whereas TGCN's embeddings remain consistent for nodes belonging to the same class. From a node-centric perspective, TGCN's embeddings seem reasonable. Nodes of the same class possess identical features and exhibit the same topological structure. Therefore, it is logical for them to share a common embedding. However, our embeddings tend to differentiate each individual node. We believe that this characteristic plays a role in mitigating the oversmoothing problem within our model.

Furthermore, we conducted an additional experiment to quantitatively evaluate our model’s ability to tackle over-smoothing. This experiment is conducted on the binary toy dataset: the toy data we constructed consists of a dynamic graph with a maximum of 64 snapshots (adapt to layers), with each snapshot being a clique of 10 nodes. Each node has 10 associated attributes. The task is binary classification where each node’s class information is hidden in its first time-stamp’s attributes. Attributes of other time stamps are randomly sampled from the normal distribution. All methods are trained with a maximum of 2000 epochs and a learning rate of 0.001. Finally, we evaluate their smoothness using the Mean Average Distance (MAD). Results are summarized in Table \ref{tab:smooth}.
\begin{table}[]
    \centering
    \begin{tabular}{c|c|c|c}
    \hline
        Layers&	GCN-GRU&	T-GCN	&IDGNN\\
        8	&0.6217	&0.9934	&1.1113\\
        16	&0.5204	&0.8719	&1.0982\\
        32	&0.0077	&0.7176	&1.0019\\\hline
    \end{tabular}
    \caption{Smoothness of embeddings. (The larger the better)}
    \label{tab:smooth}
\end{table}

\begin{figure}
     \centering
     \begin{subfigure}[b]{0.35\textwidth}
         \centering
         \includegraphics[width=\textwidth]{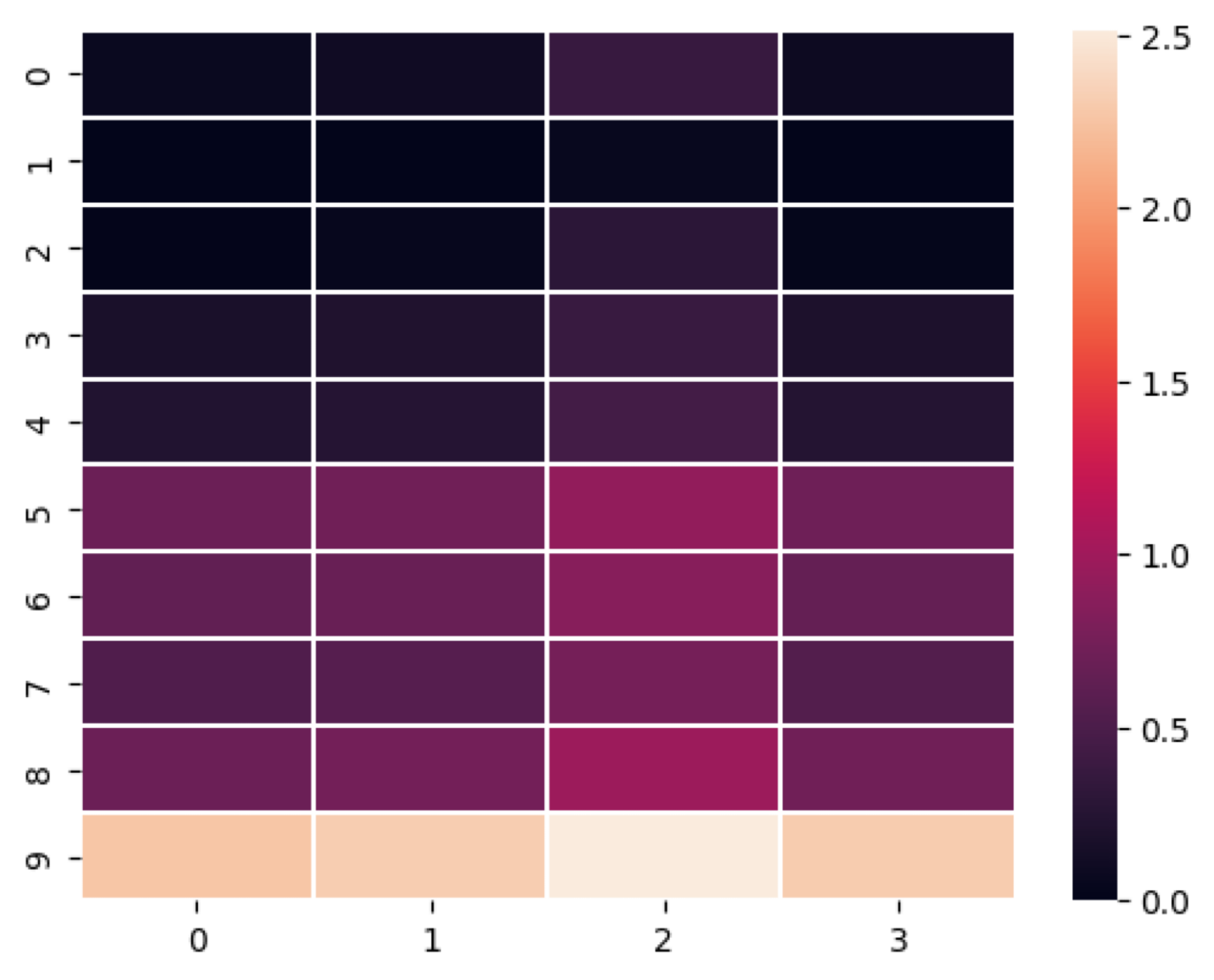}
         \caption{Our model}
         \label{fig1}
     \end{subfigure}
     \hfill
     \begin{subfigure}[b]{0.35\textwidth}
         \centering
         \includegraphics[width=\textwidth]{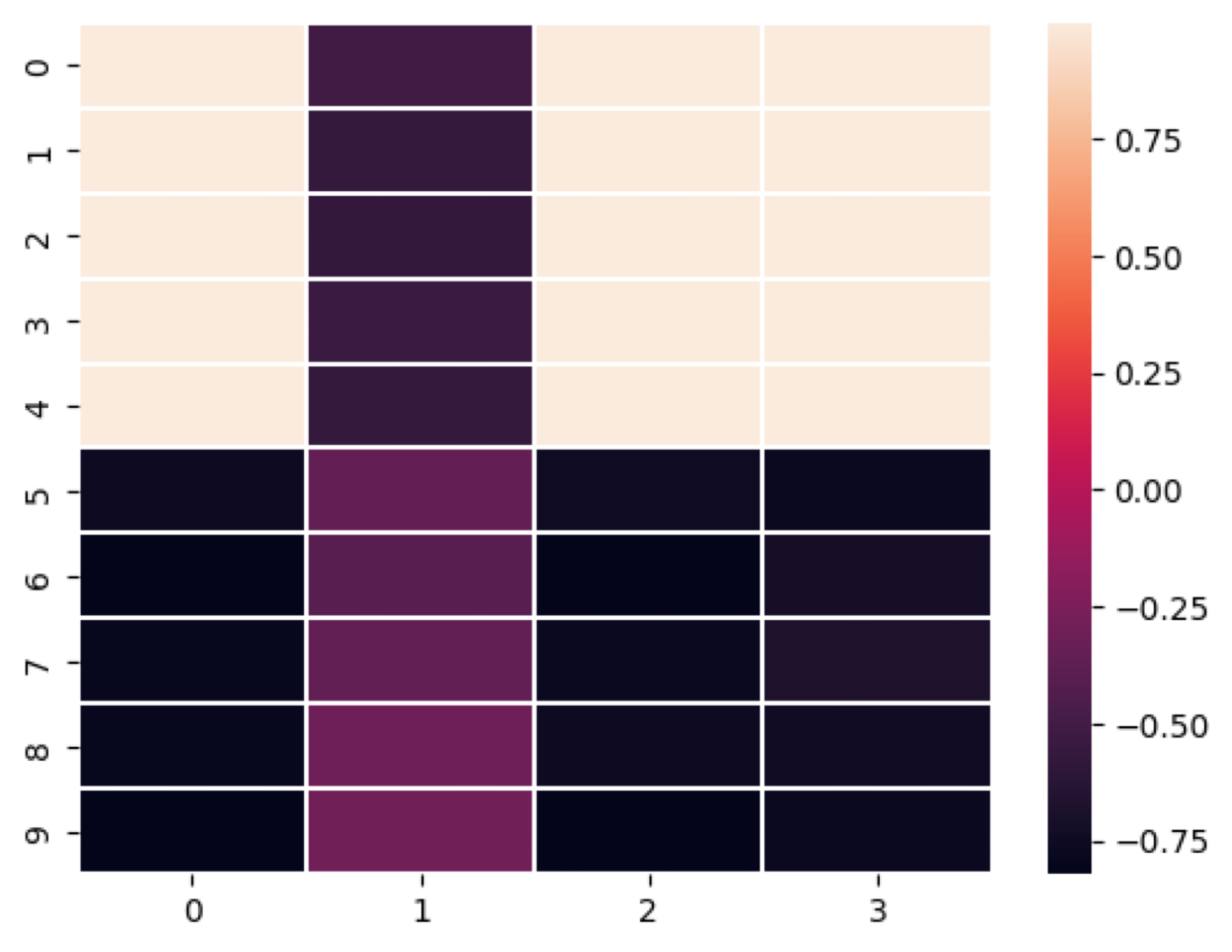}
         \caption{TGCN}
         \label{fig2}
     \end{subfigure}
        \caption{The embedding visualization of our method and TGCN}
        \label{fig:two graphs}
\end{figure}


\section{Hessian Vector Product}
To compute the product of Hassian and a vector: $Hv$, and $H=\frac{\partial^2 f}{\partial x^2}$. We compute the product by $Hv = \frac{\partial (\frac{\partial f}{\partial x })^T v }{\partial x}$. In this way, we are not explicitly computing the Hessian. 
\footnote{\url{https://jax.readthedocs.io/en/latest/notebooks/autodiff_cookbook.html}}

\section{Additional Result on Toy Data}
\begin{figure}
    \centering
    \includegraphics[width=.35\textwidth]{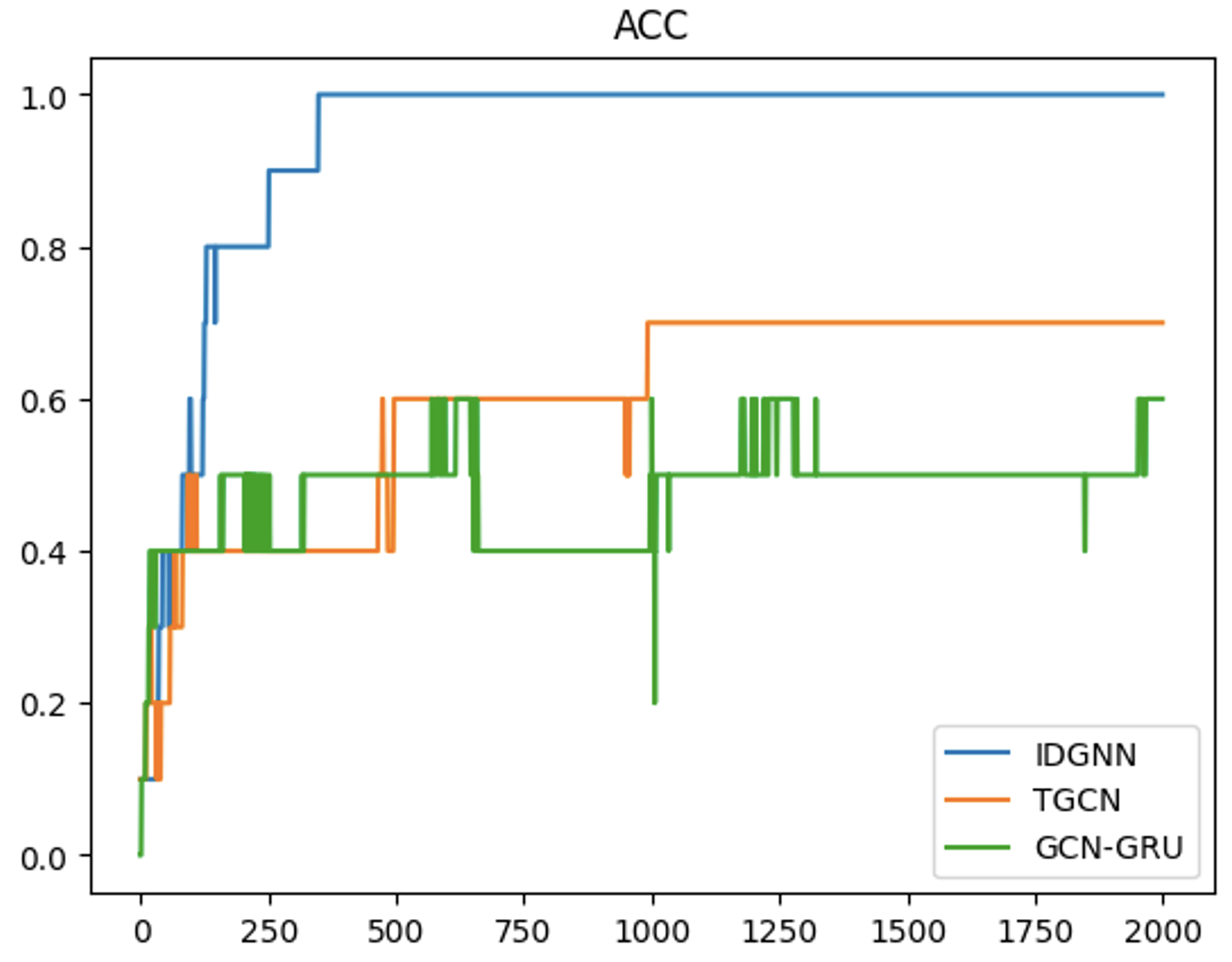}
    \caption{Additional result.}
\end{figure}
This synthetic experiment shifts label information from time stamp 1 to time stamp 5. This adjustment ensures uniform difficulty in utilizing label information across all models. Implementing this change postpones our method towards achieving 100\% accuracy by approximately 50 epochs. However, even with this modification, the baselines still struggle to fit the data.


\end{document}